\documentclass[a4paper,11pt,oneside]{article}

\usepackage[english]{babel}
\usepackage[T1]{fontenc}
\usepackage[utf8]{inputenc}

\usepackage[pdftex,unicode]{hyperref}
\hypersetup{pdftitle=Clustering with Penalty for Joint Occurrence of Objects: Computational Aspects}
\hypersetup{pdfauthor=Ondřej Sokol and Vladimír Holý}

\usepackage{color}
\definecolor{myred}{rgb}{0.5,0.0,0.0}
\hypersetup{colorlinks=true, linkcolor=myred, anchorcolor=myred, citecolor=myred, filecolor=myred, urlcolor=myred}
\definecolor{light-gray}{gray}{0.65}
\newcommand{\GC}[1]{\textcolor{light-gray}{#1}}

\usepackage[margin=60pt]{geometry}
\usepackage{amsmath}
\usepackage{amssymb}
\usepackage{amsthm}
\usepackage{bm}
\usepackage{mathtools}

\usepackage{enumitem}

\usepackage{graphicx}
\usepackage{float}
\usepackage{hhline}
\usepackage{multirow}
\usepackage{booktabs}

\newtheorem{theorem}{Theorem}

\newtheorem*{remark}{Remark}
\newtheorem{problem}{Problem}

\usepackage{algorithm}
\usepackage{algorithmic}
\usepackage{listings}

\usepackage[authoryear]{natbib}

\begin{document}

\begin{center}
{\Large \bfseries Clustering with Penalty for Joint Occurrence of Objects: Computational Aspects}\footnote{Preliminary results were presented in \cite{Sokol2020}.}
\end{center}

\begin{center}
{\bfseries Ondřej Sokol} \\
Prague University of Economics and Business \\
Winston Churchill Square 4, 130 67 Prague 3, Czechia \\
\href{mailto:ondrej.sokol@vse.cz}{ondrej.sokol@vse.cz} \\
Corresponding Author
\end{center}

\begin{center}
{\bfseries Vladimír Holý} \\
Prague University of Economics and Business \\
Winston Churchill Square 4, 130 67 Prague 3, Czechia \\
\href{mailto:vladimir.holy@vse.cz}{vladimir.holy@vse.cz} \\
\end{center}

\noindent
\textbf{Abstract:}
The method of Holý, Sokol and Černý (Applied Soft Computing, 2017, Vol. 60, p. 752--762) clusters objects based on their incidence in a large number of given sets. The idea is to minimize the occurrence of multiple objects from the same cluster in the same set. In the current paper, we study computational aspects of the method. First, we prove that the problem of finding the optimal clustering is NP-hard. Second, to numerically find a suitable clustering, we propose to use the genetic algorithm augmented by a renumbering procedure, a fast task-specific local search heuristic and an initial solution based on a simplified model. Third, in a simulation study, we demonstrate that our improvements of the standard genetic algorithm significantly enhance its computational performance.
\\

\\

\\

\section{Introduction}
\label{sec:intro}

Clustering of objects is typically based on a distance between objects or a density of objects in an area. \cite{Holy2017} propose a very different approach and cluster objects based on their joint occurrence in observed sets. It is assumed that there should typically be at most one object from each cluster in a single observed set. Deviation from this behavior is considered an error and the goal is to find the clustering which minimizes this error. Specifically, \cite{Holy2017} define the error as the average ratio of clustering decisions in which two objects from the same cluster occur in the same set. The problem of finding clusters minimizing this error can be formulated as an integer nonlinear optimization problem.

The motivation for this clustering technique lies in retail analytics. \cite{Holy2017} use this method to cluster products of a retail store into categories of substitutes. In this setting, it is assumed that most customers buy at most one product (object) from each cluster in a single visit (set), i.e. two or more products (objects) from the same cluster rarely occur together in the same shopping basket (set). This is quite reasonable behavior suggesting customers choose only one product for a given purpose. In retail, there is typically a large number of products (objects) and a huge number of shopping baskets (sets) making the method computationally very intensive. The main advantage is that no characteristics of products are needed and only a history of transactions is utilized. \cite{Holy2017} show that this method can uncover meaningful clusters in an empirical study of a Czech  drugstore chain. For other clustering approaches in retail business, see e.g.\ \cite{Jonker2004}, \cite{Tsai2004}, \cite{Reutterer2006}, \cite{Zhang2007}, \cite{Lingras2014}, \cite{Ammar2016}, \cite{Peker2017}, \cite{Wu2020}, and \cite{Sokol2019a}.

In this paper, we analyze the approach of \cite{Holy2017} from a computational point of view. First, we demonstrate that the related optimization problem is NP-hard. We build our proof on the results of \cite{Karp1972} for the Max-Cut problem. Second, we propose to numerically find clusters using the genetic algorithm combined with local search. We adjust the standard genetic algorithm by applying the renumbering procedure of \cite{Falkenauer1998} and \cite{Hruschka2003} suitable for integer representation of clusters. Note that \cite{Hruschka2003} also propose crossover and mutation operations specifically adapted for the integer representation of clusters. However, these operations are based on a distance between objects and are not applicable in our case. We therefore resort to the standard versions of these operations. Such approach in the context of clustering is used e.g. by \cite{Murthy1996}. We further enhance the genetic algorithm by computationally effective local search. This improvement of the genetic algorithm is in general suggested e.g.\ by \cite{Hamzacebi2008}. Finally, in the initial population, we include the solution of the simplified problem obtained by the k-means method with data transformed to distances. Our modifications significantly improve the computational performance in comparison to the basic genetic algorithm utilized by \cite{Holy2017}. Our numerical method falls to the field of clustering methods based on nature-inspired metaheuristics. For an overview of this field, see e.g.\ \cite{Hruschka2009}, \cite{Nanda2014} and \cite{Jose-Garcia2016}.

The rest of the paper is structured as follows. In Section \ref{sec:problem}, we formulate our problem of finding clusters. In Section \ref{sec:complex}, we prove that this problem is NP-hard. In Section \ref{sec:opti}, we propose to numerically solve this problem by the improved genetic algorithm. In Section \ref{sec:sim}, we investigate the computational performance of the proposed algorithm using simulated data. We conclude the paper in Section \ref{sec:con}.

\section{Problem Statement}
\label{sec:problem}

\subsection{Integer Formulation}
\label{sec:problemInteger}

The problem is based on the following data structure. Let $N$ be the number of sets, $M$ the number of objects and $K$ the maximum number of clusters. A matrix $\mathbf{A}$ is available with $N$ rows, $M$ columns and elements $a_{ij}$ defined as
\begin{equation}
a_{ij} = \begin{cases}
1 & \text{if object } j \text{ is present in set } i, \\
0 & \text{otherwise}.
\end{cases}
\end{equation}
Furthermore, we assume that there are at least two objects in each set, i.e.\ $\sum_{j = 1}^{M} a_{ij} \geq 2$ for all $i \in \{1,\dots,N \}$. Otherwise the set is not interesting for our goal and can be ignored. 

The variables in the model are vectors of possible object clustering $\bm{x}=(x_1,\ldots,x_{M})'$, where $x_j$ is an integer in range $1 \leq x_j \leq K$ for every $j \in \{1,\dots,M \}$. We are looking for such clustering $\bm{x}$ that minimizes the weighted occurrences of pairs of objects from one cluster in the same set.

For each set $i=1,\ldots,N$ we denote the total number of object pairs in the set $D_i$ as
\begin{equation}
\label{eqCost2}
D_i=\binom{E_i}{2}, \qquad E_i=\sum_{j=1}^{M}a_{ij},
\end{equation}
and the number of violating object pairs from one cluster within the same set $V_i$ as
\begin{equation}
\label{eqCost1}
V_i (\bm{x})= \sum_{k = 1}^{K} \binom{W_{ik}(\bm{x})}{2}, \qquad W_{ik}(\bm{x})=\sum_{j=1}^{M} a_{ij} \mathbb{I} ( x_j = c ),
\end{equation}
where $ \mathbb{I}(\cdot)$ denotes the indicator function. 
The cost function is then
\begin{equation}
\label{eqCost}
f_{cost} (\bm{x}) = \frac{1}{N} \sum_{i=1}^{N} \frac{V_i(\bm{x})}{D_i}.
\end{equation}
Therefore, the cost function equals to the average ratio of object pairs in which two object from the same cluster are in the same set. The range of the cost function is from 0 (there is no set containing object from the same cluster) to 1 (every set contain only objects from the same cluster).

The nonlinear integer programming problem is of form
\begin{equation}
\label{eqOpti}
\begin{aligned}
\min_{\bm{x}} \quad & f_{cost} (\bm{x}) \\
\text{s. t.} \quad & x_j \leq K & \text{for } j=1,\ldots,M, \\
& x_j \in \mathbb{N} & \text{for } j=1,\ldots,M. \\
\end{aligned}
\end{equation}
In practical tasks, we can assume $ N \gg M \gg K $.


\subsection{Binary Formulation}
\label{sec:problemBinary}

The model can be straightforwardly transformed from the integer program to a binary program. Let 
\begin{equation}
y_{jk} = \begin{cases}
1 & \text{if object } j \text{ is assigned to cluster } k, \\
0 & \text{otherwise},
\end{cases}
\end{equation}
for each $j$ and $k$ and let $\bm{Y}$ be matrix with elements $y_{jk}$. Note that in the optimization process itself, the number of variables $y_{jk}$ can be further reduced by $M$ as $y_{jK} = 1 - \sum_{k = 1}^{K-1} y_{jk}$ for every $j = 1,\dots,M$. 
As the number of violating object pairs is dependent on clustering $\bm{Y}$, we define $V'_i$ alternatively as
\begin{equation}
\label{eqCost1bin}
V'_i (\bm{Y})= \sum_{k=1}^{K} \binom{W'_{ik}(\bm{Y})}{2}, \qquad W'_{ik}(\bm{Y})=\sum_{j=1}^{M} a_{ij} y_{jc}.
\end{equation}
Similarly to \eqref{eqOpti}, we define the cost function as
\begin{equation}
\label{eqCostbin}
f'_{cost} (\bm{Y}) = \frac{1}{N} \sum_{i=1}^{N} \frac{V'_i(\bm{Y})}{D_i}.
\end{equation}
The non-linear binary programming model is then 
\begin{equation}
\label{eqOptibin}
\begin{aligned}
\min_{\bm{Y}} \quad & f'_{cost} (\bm{Y}) \\
\text{s. t.} \quad & \sum_{k=1}^{K} y_{jk} = 1 && \text{for } j=1,\ldots,M, \\
& y_{jk} \in \{0,1\} && \text{for } j=1,\ldots,M, \quad k = 1,\ldots,K. \\
\end{aligned}
\end{equation}

\section{Computational Complexity}
\label{sec:complex}

\begin{theorem}
Problem \eqref{eqOpti} is NP-hard.
\end{theorem}

\begin{proof}
In order to prove it, we reduce problem \eqref{eqOptibin} to the simplest case. Set $K = 2$, i.e.\ let there be only two clusters of objects, and let $\sum_{j=1}^{M} a_{ij} = 2$ for all $i$, i.e.\ the size of all sets is 2. The number of object pairs is therefore $D_i = 1$ in each set $i$. Also, as we have only two clusters, we can define 
\begin{align}
z_{j} &= \begin{cases}
1 & \text{if object } j \text{ is assigned to cluster 1}, \\
0 & \text{if object } j \text{ is assigned to cluster 2},
\end{cases} 
\end{align}
and let
\begin{equation}
P_{j\ell} \coloneqq \sum_{i=1}^{N} a_{ij} a_{i\ell} \quad \text{for } j,\ell=1,\ldots,M, \quad j \neq \ell,
\end{equation}
which is the number of occurrences of every pair of objects in the same set. 

The problem can now be formulated as 
\begin{equation}
\label{eqOptibin2}
\begin{aligned}
\min_{z} \quad  \sum_{\forall j<\ell} P_{j\ell} \mathbb{I} ( z_j = z_\ell )  &&  \\
\text{s. t.}  \hspace{1em} z_{j} \in \{0,1\} \hspace{1em} \text{for } j=1,\ldots,M.
\end{aligned}
\end{equation}
Without the loss of generality, we can rewrite the objective function to the maximization form as
\begin{equation}
\label{eqOptibin3}
\begin{aligned}
\max_{z} \quad  \sum_{\forall j<\ell} P_{j\ell} \mathbb{I} ( z_j \neq z_\ell )  &&  \\
\text{s. t.}  \hspace{1em} z_{j} \in \{0,1\} \hspace{1em} \text{for } j=1,\ldots,M.  
\end{aligned}
\end{equation}
Let $\mathbf{P}$ denote matrix with elements $P_{j\ell}$. The decision problem \eqref{eqOptibin3} can be formulated as follows:

\begin{problem}
\label{problem1}
Is there a binary vector $z = \{ z_1, z_2, \dots, z_{M}  \}$ such that for a given $M, \mathbf{P}$, and $C$ holds
\begin{equation}
\label{decisionMod}
\sum_{\forall j<\ell} P_{j\ell} \mathbb{I} ( z_j \neq z_\ell ) \geq C?
\end{equation}
\end{problem}

Therefore, an instance of the problem is given by $\{ M, \mathbf{P},C \}$.
Now, we show that the Max-Cut problem can be reduced to Problem \eqref{decisionMod}.
The Max-Cut is an NP-hard decision problem (see the proof of NP-hardness in \citealp{Karp1972}) defined as follows: 

\begin{problem}
\label{problem2}
Having a graph $G = (V_G,E_G)$, weighting function $w: E_G \rightarrow \mathbb{Z}$ and a positive integer $C$, is there a set $S \subset V_G$ such that
\begin{equation}
\label{decisionMaxCut}
\sum_{\{j,\ell\} \in V_G, j \in S, \ell \notin S } w(\{j,\ell\}) \geq C?
\end{equation}
\end{problem}

Hence, an instance of the Max-Cut problem is given by $\{ V_G, E_G, w, C \}$.
In order to prove that the Max-Cut problem is reducible to Problem \ref{problem1}, we need to shown that every instance of Problem \ref{problem2} is reducible to Problem \ref{problem1}. The reduction function is following:  
\begin{equation} 
g : \{V_G, E_G, w, C\} \mapsto \{n,\mathbf{P},C\}. 
\end{equation} 
$V_G$ is mapped to a vector $(1,2,\dots,n)$ where $n$ is the number of vertices. Values of weighting function $w(j,\ell)$ are directly translated into $P_{j\ell}$. If graph $G = (V_G,E_G)$ is not complete, then $P_{j\ell}$ is set to zero for missing edges. A value of $C$ remains the same.

\begin{remark}
Note that the Max-Cut problem does not assume non-negativity weighting function $w$ and in our setup $P_{j\ell}$ are naturally non-negative as it is number of instances that two objects are in the same set; however, this is not a problem. As it is shown in reduction of the Knapsack problem to the Max-Cut problem through the Partition problem \citep{Karp1972}, there are two cases when the $w(j,\ell)$ are negative.
\begin{enumerate}
\item If the sum of all items in the Knapsack form is lower than the capacity, then $w(j,\ell)$ may be negative. However, such instance is trivial to solve.
\item If the weight of $i$-th item $w^{(KP)}_j$ in the Knapsack problem is negative, then $w(j,\ell)$ may be negative. However, this case can be easily transformed to instance without negative weights. Such transformation consists in multiplication of negative weight $w^{(KP)}_j$ by $-1$, switching affected binary variable $x_j^{(KP)}$ to $1-x_j^{(KP)}$ and increase in the total capacity of knapsack by $-w^{(KP)}_i$.
\end{enumerate}
Therefore, the cases in which $w(j,\ell)$ are negative can be straightforwardly transformed to non-negative cases or are easy to solve.
\end{remark}

As a result, every \textit{hard} instance of the Max-Cut problem given by $[V_G, E_G, w, C]$ can be transformed to Problem \ref{problem1} using reduction function $g$ and therefore problem \eqref{eqOpti} is NP-hard.
\end{proof}

\section{Numerical Optimization}
\label{sec:opti}

\subsection{Genetic Algorithm}
\label{sec:optiGen}

In our model, it is not possible to calculate the similarity between individual clusterings without losing important information. Hence, we cannot use standard clustering methods such as k-means, DBSCAN or hierarchical clustering methods. For this reason, an integer genetic algorithm is chosen to find a suitable clustering.

The standard genetic algorithm has the following phases. First, the initial population is generated, where by population we mean a set of different clusterings $\bm{x}$ called individuals. Objects $x_j$ assigned to clusters in clustering $\bm{x}$ are called genes. Individuals can also be inserted to population if promising candidates are available.
After the initial population is prepared, in each generation following steps are conducted:
\begin{enumerate}
\item
Designation of given number of the best individuals as elite in order to preserve them.
\item
Modification of non-elite individuals.
\begin{enumerate}
\item
By crossover, when two different individuals are randomly selected using roulette wheel selection and by randomly swapping genes two new individuals are created, replacing the original ones.
\item
By mutation, when randomly selected genes are randomly switched.
\end{enumerate}
\item
Evaluation of all new individuals.
\item
If the specified number of generations or a sufficiently high quality result is not achieved, the algorithm returns to Step 1, else the best individual is chosen and the algorithm ends.
\end{enumerate}

In order to apply the genetic algorithm, several parameters are to be chosen. The first one is the \textit{size of population}. With bigger population of individuals more possible clusterings are explored, which can result in finding better solution at the cost of higher computational demands. This parameter is task-specific and in our case a large population is preferred (see \citealp{Holy2017}). The \textit{number of iterations}, sometimes called generations in terms of evolutionary algorithms, is another parameter. Simply put, it is a parameter of how long the solution space should be searched. The\textit{ number of elites} parameter determines how many individuals with the best evaluation are declared elite and are passed to the next generation without any alteration. This prevents the loss of the best individuals in population. Parameter \textit{mutation chance} determines the probability of changing gene value to a random one. The purpose of mutation in genetic algorithm is to introduce more diversity into population, thus to avoid reaching local minimum by preventing the individuals from becoming too similar to each other. However, if a high value is selected, then the crossover effect is suppressed and algorithm is more of a random search of the solution space. 

By far the most computationally demanding part of the algorithm is evaluating newly found clusters as it is necessary to work with a matrix with large dimensions, specifically $ N \times M $.

\subsection{Renumbering Procedure}
\label{sec:optiRenum}

There are several complications concerning the application of traditional genetic algorithms for clustering tasks. The main one is the clustering codification. In our case, the clustering  $ \bm{x} $ has $M$ elements $x_j$ with values $ \{1, \dots, K \} $. If the standard genetic algorithm is used, the resulting clustering $\bm{x}$ allows symmetries in the solution space, for example, assuming $M=5$ and $K=3$, solution $\bm{x} = (3, 3, 2, 1, 3)$ is in fact identical to $\bm{x} = (2, 2, 1, 3, 2)$ but standard genetic algorithm treats them as significantly different. This has inappropriate consequences, especially in the case of crossover. This shortcoming can be remedied by introducing a simple rule (see \citealp{Falkenauer1998}) when cluster numbers are renumbered to start from the smallest based on the first occasion of each cluster, e.g.\ the solution above would be renumbered to $\bm{x} = (1, 1, 2, 3, 1)$. The renumbering procedure allows the suitable application of classic genetic algorithms in clustering problems, avoiding the problems of redundant codification (see \citealp{Hruschka2003}).

\subsection{Local Search}
\label{sec:optiLoc}

Second proposed improvement is the implementation of a task-specific local search. The local search function consists in generating all possible \textit{neighbor} individuals for the best individual in each iteration and checking whether the new individuals improve the cost function. If newly found individual gives better result than the original one, then the original is simply replaced. If the individual has already been checked in the previous iterations, then no local search is executed as the individual cannot be improved by local search. 
 
By far the most time-consuming part of used genetic algorithm is the frequently called enumeration of $f_{cost} (\bm{x})$. In the proposed procedure, we therefore try to approach enumeration efficiently by repeated usage of the intermediate calculations, similar to dynamic programming approach. 
The goal is to find neighbors of individual $ \bm{x} $ , i.e.\ all possible vectors $ \bm{x}^{(j,k_1)} $ for all $j$ and $k_1$, which differs from $ \bm{x} $ in exactly one element $j$ which is changed from the original cluster $k_0$ to the new cluster $k_1$. 
To find neighbors, it is not necessary to calculate the value of the cost function from the beginning, but to use already prepared calculations. The main idea follows from the decomposition: 
\begin{align} 
V(\bm{x}^{(j,k_1)}) := \sum_{i=1}^{N} V_i(\bm{x}^{(j,k_1)}) =& \sum_{i=1}^{N} V_i(\bm{x}) \nonumber \\ 
&+ \sum_{i=1}^{N} \left(\binom{v_{ik_0}(\bm{x}^{(j,k_1)})}{2} - \binom{v_{ik_0}(\bm{x})}{2} \right)   \nonumber \\ 
&+  \sum_{i=1}^{N} \left( \binom{v_{ik_1}(\bm{x}^{(j,k_1)})}{2} - \binom{v_{ik_1}(\bm{x})}{2} \right)  \nonumber\\
=& \sum_{i=1}^{N} V_i(\bm{x}) \label{local:eq} \\ 
&+ \underbrace{\sum_{i=1}^{N} \binom{v_{ik_0}(\bm{x}^{(j,k_1)})}{2}}_{V^{(0)}({\bm{x}^{(j,k_1)}})} - \underbrace{\sum_{i=1}^{N} \binom{v_{ik_0}(\bm{x})}{2}}_{V^{(0)}({\bm{x}})}    \nonumber \\ 
&+    \underbrace{\sum_{i=1}^{N} \binom{v_{ik_1}(\bm{x}^{(j,k_1)})}{2}}_{V^{(1)}({\bm{x}^{(j,k_1)}})} - \underbrace{\sum_{i=1}^{N} \binom{v_{ik_1}(\bm{x})}{2}}_{V^{(1)}({\bm{x}})} \nonumber
\end{align}
With the knowledge of $V(\bm{x}^{(j,k_1)})$, the cost function $f_{cost} (\bm{x}^{(j,k_1)})$ can be computed quickly as the other parts of the formula remains unchanged.

Let $\bm{D}$ is a vector of all $D_i$ and $\mathbf{B}$ is the matrix $N \times K$ of the number of objects by set and cluster. In the algorithm, the operation $/$ stands for element-wise division and $\textbackslash$ is a set subtraction operation.
The local search function which finds all neighbors of clustering $\bm{x}$ is described in Algorithm \ref{alg:LocalSearch}. 
\begin{algorithm}[H]
    \begin{algorithmic}[1]
        \normalsize
        \REQUIRE{$\bm{x}, V(\bm{x}),\mathbf{A},\mathbf{B},\bm{D}$}
        \STATE{$V_{best} := V(\bm{x})$}
        \STATE{$\bm{x}_{best} := \bm{x}$}
        \FOR{$j \in \{1,\dots,M\}$}
             \STATE{$ \bm{j_0} := \textsc{which}(\mathbf{A}[,j]=1)$}
             \STATE{$ k_0 := \bm{x}[j]$}
             \STATE{$\bm{B}^{(0)} := \mathbf{B}[,k_0]$}
             \IF {$V^{(0)}({\bm{x}}) = \textsc{NULL}$} 
                  \STATE{$V^{(0)}({\bm{x}}) := \textsc{sum}(\textsc{choose}(\bm{B}^{(0)},2)/\bm{D})$}
             \ENDIF      
             \STATE{$\bm{B}^{(0)}[\bm{j_0}] := \bm{B}^{(0)}[\bm{j_0}]-1$}
             \STATE{$V^{(0)}({\bm{x}^{(j,k_1)}}) := \textsc{sum}(\textsc{choose}(\bm{B}^{(0)},2)/\bm{D})$}                             
             \FOR{$k_1 \in \{1,\dots,K\} \ \textbackslash \ k_0$} 
                  \STATE{$\bm{B}^{(1)} := \mathbf{B}[,k_1]$}
                  \STATE{$V^{(1)}({\bm{x}}) := \textsc{sum}(\textsc{choose}(\bm{B}^{(1)},2)/\bm{D})$}      
                  \STATE{$\bm{B}^{(1)}[\bm{j_0}] := \bm{M}^{(1)}[\bm{j_0}]+1$}
                  \STATE{$V^{(1)}({\bm{x}^{(j,k_1)}}) := \textsc{sum}(\textsc{choose}(\bm{B}^{(1)},2)/\bm{D})$}       
                  \IF {$(V(\bm{x}) - V^{(0)}({\bm{x}}) + V^{(0)}({\bm{x}^{(j,k_1)}}) - V^{(1)}({\bm{x}}) + V^{(1)}({\bm{x}^{(j,k_1)}}))  < V_{best})$}
                    \STATE{$V_{best} := V(\bm{x}) - V^{(0)}({\bm{x}}) + V^{(0)}({\bm{x}^{(j,k_1)}}) - V^{(1)}({\bm{x}}) + V^{(1)}({\bm{x}^{(j,k_1)}})$}
                    \STATE{$\bm{x}_{best} := \bm{x}$}
                    \STATE{$\bm{x}_{best}[j] := k_1$  }      
                    \ENDIF        
            \ENDFOR              
        \ENDFOR
        
        \RETURN $\bm{x}_{best}, V_{best}$
    \end{algorithmic}
    \caption{Local Search function}
    \label{alg:LocalSearch}
\end{algorithm}


The proposed local search function is significantly faster than naive neighbor generation and repeated cost function calls. A naive procedure would call the cost function $ M \cdot (K - 1) $ times as the cluster of each object can be changed. The enumeration complexity of the cost function is of the order of $ N \cdot M $. The total computational complexity of generating and enumerating of all neighbors is then $ \mathcal{O} (M \cdot K \cdot N \cdot M) = \mathcal{O} (M ^ 2 \cdot K \cdot N) $.

The proposed method also evaluates all neighbors, i.e.\ $ M \cdot (K - 1) $ clusterings, but instead of repeated calls to the cost function, intermediate calculations are used repeatedly. The enumerations of the number of object pairs from the same cluster in one set is performed for the original clustering with a complexity of the order of $ N \cdot M $ and then is used repeatedly. In the enumeration of the change in the cost function of individual neighbors, only the parts that actually differ are calculated, this is of the order of $2 \cdot N $ for one neighbor since only two sums change in respect to original clustering, see \eqref{local:eq}. The total computational complexity is then $ \mathcal{O} (N \cdot M + M \cdot (\ K - 1) \cdot N \cdot 2) = \mathcal{O} (N \cdot M + M \cdot \ K \cdot N) = \mathcal{O} (M \cdot K \cdot N) $. Since $ M $ is expected to be of high value, e.g.\ hundreds or thousands, the time saved is significant.

\subsection{Initial Solution}
\label{sec:optiInit}

Finally, we obtain an appropriate initial solution. We transform the data set to dissimilarity matrix $\mathbb{Q}$ based on the simplified relationship between objects. The elements $q_{j\ell}$ of matrix $\mathbb{Q}$ are computed as the proportions of sets in which objects appear together in the same set:
\begin{equation}
q_{j\ell} \coloneqq \frac{1}{N} \sum_{i=1}^{N} a_{ij} a_{i\ell} \quad \text{for } j,\ell=1,\ldots,M, \quad j \neq \ell,
\end{equation}
Using the dissimilarity matrix $\mathbb{Q}$, we can use the standard clustering algorithms such as k-means to find initial solution for genetic algorithm. A substantial amount of information is lost through the data transformation, nevertheless, the solution found by k-means can serve as suitable initial point. The main advantage is that k-means is a very fast method, which allows finding a suitable initial point in a matter of seconds.

\section{Simulation Study}
\label{sec:sim}

\subsection{Setup}
\label{sec:simSetup}

In the simulation study, we compare four modifications of the genetic algorithm for clustering problems in terms of the best solution found and the speed of convergence. The modifications are as follow:
\begin{enumerate}[label=(\alph*)]
\item \textit{Standard} algorithm with completely random initial population and without local search;
\item \textit{Only Local} algorithm with completely random initial population but using the proposed local search function;
\item \textit{Only Init} algorithm with inserted initial solution from the simplified model and without local search;
\item \textit{Local \& Init} algorithm with both inserted initial solution from the simplified model and using the proposed local search approach.
\end{enumerate}

The main problem of the studied clustering model is the difficulty of finding the optimal solution. Even in test instances, we cannot determine the optimum with certainty. Therefore, as a benchmark we use the solution of the simplified model found by the k-means method, as described in Section \ref{sec:optiInit} and used in \textit{Only Local} and \textit{Local \& Init} algorithms. For parameters of the k-means method we use the Euclidean distance metric with 1000 starting locations, the maximum number of iterations to 1000 and the Hartigan-Wong's algorithm. With these parameters, the k-means solution is found almost instantly.

The following parameters from \cite{Holy2017} are used for all modifications. The population size is set to 500 individuals. Initial population is generated randomly but in the case of \textit{Only Init} a \textit{Local \& Init} we also supply the k-means solution of the simplified model. We always use the fixed number of 500 \textit{generations} even though simulations show that vast majority of instances converge faster. The \textit{elite ratio} is set to 0.1. The \textit{mutation chance} is set to 0.01 to allow algorithm to concentrate on improving one point while retaining some exploratory ability of mutation. Unlike the local search which is conducted only on the best individual of the generation, the mutation can take place on any non-elite individual. 

The basic parameters for data generation are the number of sets $N$, objects $M$ and clusters $K$. Sets are generated independently. For each set, it is randomly determined how many unique clusters exist in the set. In doing so, every cluster has the same probability of occurrence given by the parameter $\rho$. From each cluster assigned to the set, one product is assigned to the set. In the next step, based on the $\pi$ parameter, other products from the same cluster are assigned into sets, thus creating situations that violate the model's assumptions.

We work with the default values $N = 10000$, $M = 100$, $K = 10$, $\rho = 0.50$ and $\pi = 0.03$ which roughly correspond to retail datasets. In the next part, we investigate the effects of changes in the values of parameters $N$, $M$, $K$, and $\pi$ on the results of the genetic algorithm modifications.

\subsection{Results}
\label{sec:simRes}

Throughout this section, we assume the parameters are set to their default value unless said otherwise -- then only one parameter changes at a time. Each algorithm is run 10 times for each dataset and each dataset is generated 100 times for each scenario. Computation time is reported for 3.40 GHz CPU and the algorithm is implemented in R software. We consider the k-means solution of the simplified model from Section \ref{sec:optiInit} as the \emph{benchmark} solution. We report values of the objective function standardized to this benchmark solution (standardized objective = objective / benchmark objective).

First, we take a look at the progress of the objective function over time in Figure \ref{fig:conv}. We report the objective function over time rather than over the number of iterations as some versions of the algorithm require to compute local search in each iteration. The computation time spent on local search is, however, quite negligible. Apparently, for default values of the parameters, the \textit{Standard} algorithm cannot overcome the benchmark solution. The local search improvement in the \textit{Only Local} algorithm allows surpassing the benchmark solution. The \textit{Only Init} and \textit{Local \& Init} algorithms use the benchmark solution as their initial point and both are able to quickly improve it with the \textit{Local \& Init} algorithm being considerably faster.

Next, we focus on differences in performance based on the changes in parameters of the data generating process and the number of clusters. The average found solutions after 500 iterations are shown in Figure \ref{fig:param} and Table \ref{tab:sim}. With a low number of sets $N$, the benchmark approach is clearly worse than the genetic algorithms. With an increasing number of sets $N$, the \textit{Standard} algorithm results are worsening and the differences between the benchmark approach and other genetic algorithms shrink. However, the \textit{Only Init} and \textit{Local \& Init} algorithms are still able to improve their initial solution. With an increasing number of objects $M$ in the dataset, the benchmark solution proves to be insufficient -- all tested modifications of the genetic algorithm give significantly better results. The same can be said about an increasing probability of extra objects $\pi$. Simply put, the more violations occur in the dataset, the worse is the benchmark solution. Concerning the impact of the true number of clusters $K$, the genetic algorithm gives better results than the benchmark approach when $K$ is significantly lower than $M$. With increasing $K$, the \textit{Standard} algorithm ceases to be suitable. With high values of $K$, \textit{Only Local}, \textit{Only Init} and \textit{Local \& Init} algorithms give very similar results to the benchmark approach. Overall, it is highly advantageous to use both proposed improvements of the genetic algorithm. Not only the \textit{Local \& Init} algorithm ends in the best solution among the four candidates, but is also the fastest one to reach it.

Finally, we examine variation in results from the perspective of repeated data generation and repeated algorithm runs. In Table \ref{tab:sim}, we report two kinds of standard deviations. The standard deviation capturing repeated data generation (labeled as SD/Sim) is obtained by first averaging objective values of all runs based on the same dataset and then taking standard deviation over all generated datasets. Conversely, the standard deviation capturing repeated algorithm runs (labeled as SD/Alg) is obtained by first taking standard deviation of all runs based on the same dataset and then averaging it over all generated datasets. We can see that SD/Alg is quite small except for the \textit{Standard} algorithm in some scenarios. Nevertheless, in most cases, random data generation is the dominant source of variation.

\begin{figure}
\centering
\includegraphics[width=0.9\textwidth]{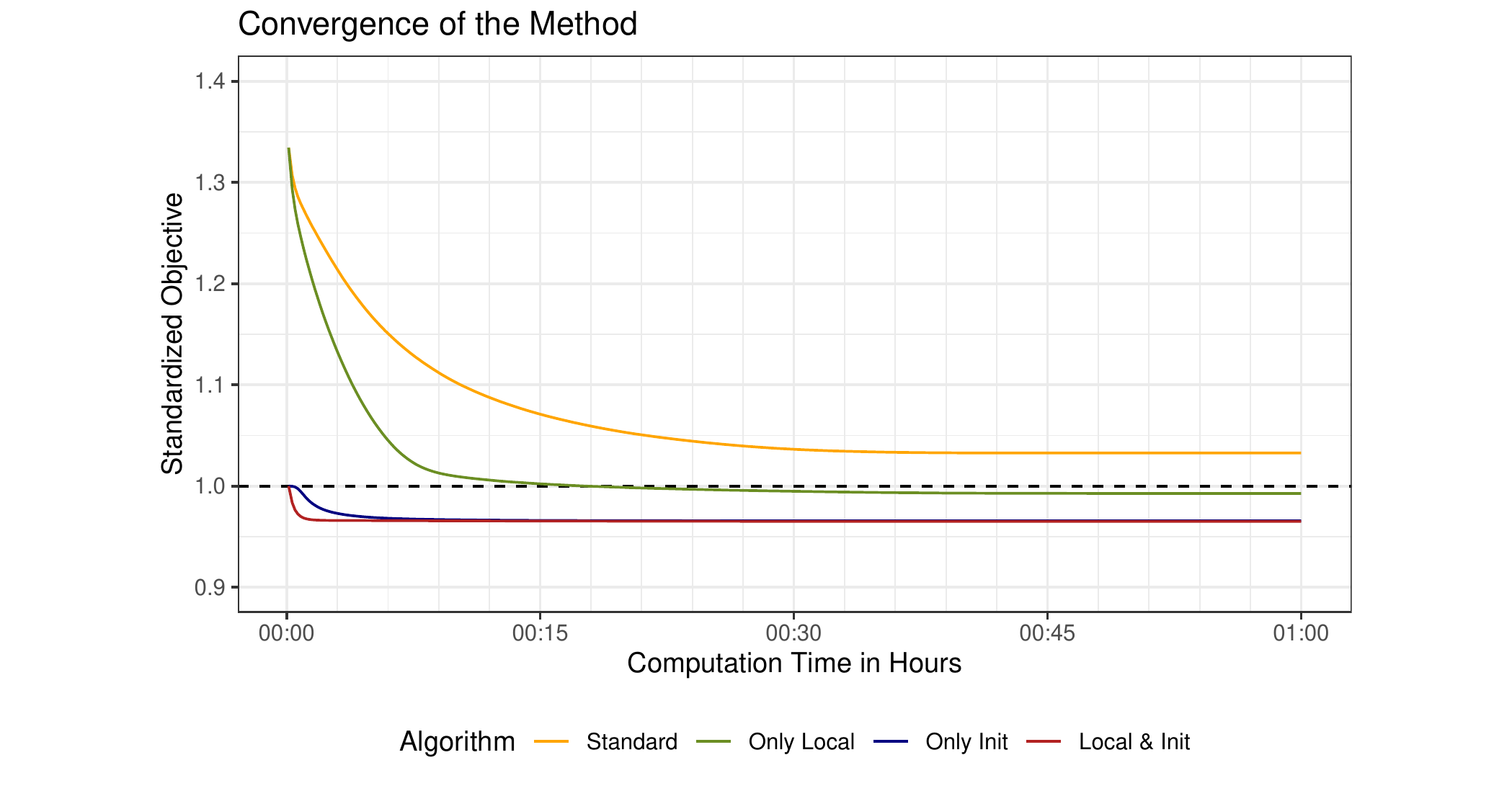}
\caption{Progress of the standardized objective function over time.}
\label{fig:conv}
\end{figure}

\begin{figure}
\centering
\includegraphics[width=0.9\textwidth]{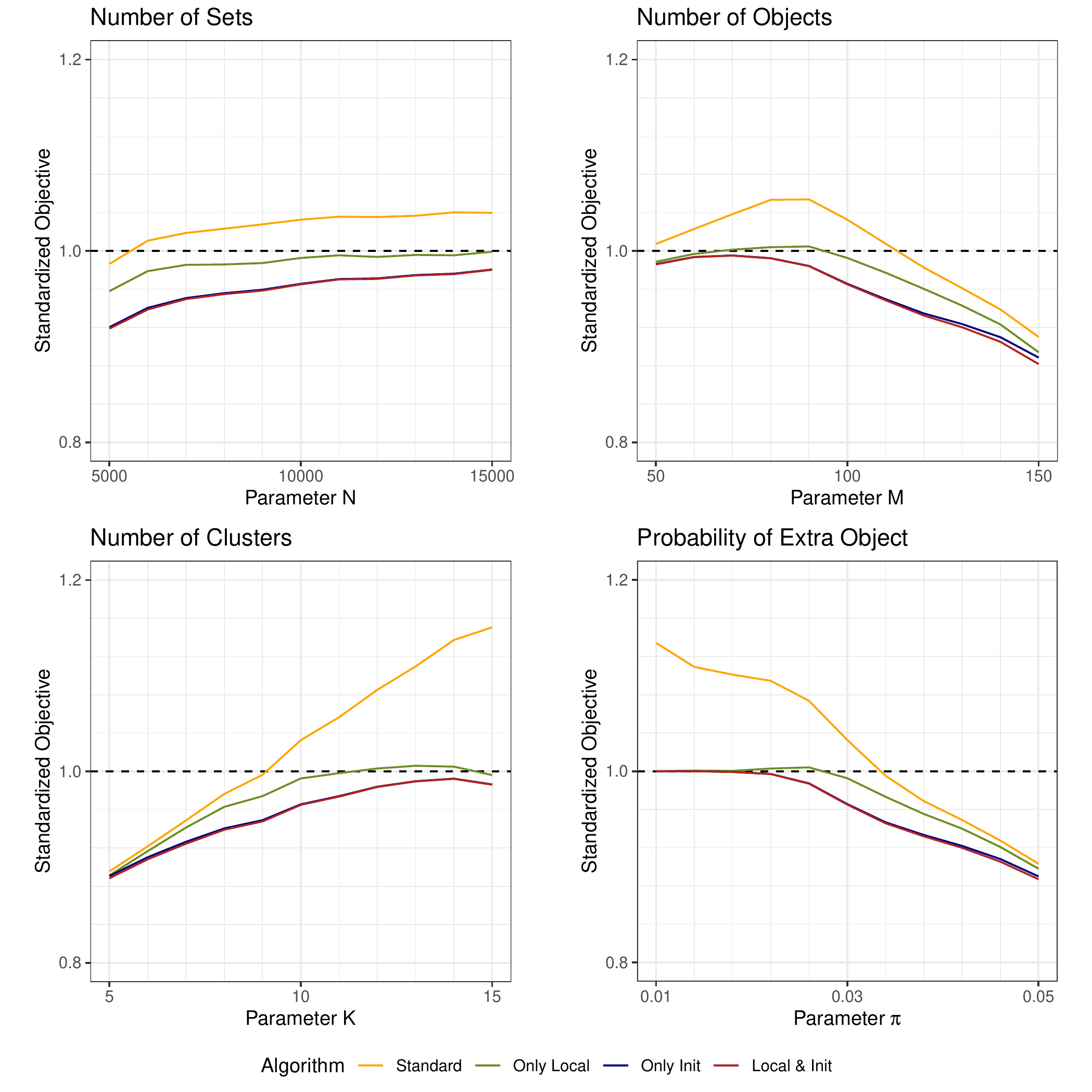}
\caption{Impact of the data generating process.}
\label{fig:param}
\end{figure}

\begin{table}
\centering
\caption{Mean values with standard deviations of the standardized objective for several scenarios of the data generating process.}
\label{tab:sim}
\begin{tabular}{lrrrrrrr}
\toprule
& \multicolumn{4}{c}{Simulation Scenario} & \multicolumn{3}{c}{Standardized Objective} \\
Algorithm & $N$ & $M$ & $K$ & $\pi$ & Mean & SD/Sim & SD/Alg \\ 
\midrule
Standard      & \GC{10000} & \GC{100} & \GC{10} & \GC{0.03} & 1.0326 & 0.0369 & 0.0202 \\ 
              &      5000  & \GC{100} & \GC{10} & \GC{0.03} & 0.9864 & 0.0458 & 0.0188 \\ 
              &     15000  & \GC{100} & \GC{10} & \GC{0.03} & 1.0397 & 0.0281 & 0.0194 \\ 
              & \GC{10000} &      50  & \GC{10} & \GC{0.03} & 1.0073 & 0.0452 & 0.0366 \\ 
              & \GC{10000} &     150  & \GC{10} & \GC{0.03} & 0.9100 & 0.0568 & 0.0055 \\ 
              & \GC{10000} & \GC{100} &      5  & \GC{0.03} & 0.8955 & 0.0618 & 0.0049 \\ 
              & \GC{10000} & \GC{100} &     15  & \GC{0.03} & 1.1505 & 0.0571 & 0.0395 \\ 
              & \GC{10000} & \GC{100} & \GC{10} &     0.01  & 1.1341 & 0.0391 & 0.1045 \\ 
              & \GC{10000} & \GC{100} & \GC{10} &     0.05  & 0.9031 & 0.0618 & 0.0052 \\ 
\midrule
Only Local    & \GC{10000} & \GC{100} & \GC{10} & \GC{0.03} & 0.9926 & 0.0308 & 0.0179 \\ 
              &      5000  & \GC{100} & \GC{10} & \GC{0.03} & 0.9580 & 0.0436 & 0.0206 \\ 
              &     15000  & \GC{100} & \GC{10} & \GC{0.03} & 0.9991 & 0.0213 & 0.0149 \\ 
              & \GC{10000} &      50  & \GC{10} & \GC{0.03} & 0.9887 & 0.0405 & 0.0072 \\ 
              & \GC{10000} &     150  & \GC{10} & \GC{0.03} & 0.8940 & 0.0563 & 0.0066 \\ 
              & \GC{10000} & \GC{100} &      5  & \GC{0.03} & 0.8913 & 0.0616 & 0.0050 \\ 
              & \GC{10000} & \GC{100} &     15  & \GC{0.03} & 0.9962 & 0.0325 & 0.0153 \\ 
              & \GC{10000} & \GC{100} & \GC{10} &     0.01  & 1.0000 & 0.0000 & 0.0000 \\ 
              & \GC{10000} & \GC{100} & \GC{10} &     0.05  & 0.8980 & 0.0615 & 0.0061 \\ 
\midrule
Only Init     & \GC{10000} & \GC{100} & \GC{10} & \GC{0.03} & 0.9656 & 0.0255 & 0.0014 \\ 
              &      5000  & \GC{100} & \GC{10} & \GC{0.03} & 0.9203 & 0.0375 & 0.0029 \\ 
              &     15000  & \GC{100} & \GC{10} & \GC{0.03} & 0.9805 & 0.0177 & 0.0009 \\ 
              & \GC{10000} &      50  & \GC{10} & \GC{0.03} & 0.9862 & 0.0400 & 0.0004 \\ 
              & \GC{10000} &     150  & \GC{10} & \GC{0.03} & 0.8885 & 0.0504 & 0.0032 \\ 
              & \GC{10000} & \GC{100} &      5  & \GC{0.03} & 0.8906 & 0.0597 & 0.0036 \\ 
              & \GC{10000} & \GC{100} &     15  & \GC{0.03} & 0.9864 & 0.0329 & 0.0004 \\ 
              & \GC{10000} & \GC{100} & \GC{10} &     0.01  & 1.0000 & 0.0000 & 0.0000 \\ 
              & \GC{10000} & \GC{100} & \GC{10} &     0.05  & 0.8899 & 0.0556 & 0.0038 \\ 
\midrule
Local \& Init & \GC{10000} & \GC{100} & \GC{10} & \GC{0.03} & 0.9650 & 0.0257 & 0.0006 \\ 
              &      5000  & \GC{100} & \GC{10} & \GC{0.03} & 0.9188 & 0.0381 & 0.0014 \\ 
              &     15000  & \GC{100} & \GC{10} & \GC{0.03} & 0.9803 & 0.0178 & 0.0005 \\ 
              & \GC{10000} &      50  & \GC{10} & \GC{0.03} & 0.9859 & 0.0410 & 0.0001 \\ 
              & \GC{10000} &     150  & \GC{10} & \GC{0.03} & 0.8817 & 0.0522 & 0.0023 \\ 
              & \GC{10000} & \GC{100} &      5  & \GC{0.03} & 0.8883 & 0.0603 & 0.0027 \\ 
              & \GC{10000} & \GC{100} &     15  & \GC{0.03} & 0.9861 & 0.0335 & 0.0001 \\ 
              & \GC{10000} & \GC{100} & \GC{10} &     0.01  & 1.0000 & 0.0000 & 0.0000 \\ 
              & \GC{10000} & \GC{100} & \GC{10} &     0.05  & 0.8870 & 0.0567 & 0.0033 \\ 
\bottomrule
\end{tabular}
\end{table}

\section{Conclusion}
\label{sec:con}

The clustering method of \cite{Holy2017} offers a unique way of categorizing products in retail stores. The main limitation of the method lies in its computational complexity. To make the method more usable for practitioners, we revisit the algorithm finding an approximate solution and improve it in several ways. We augment the basic genetic algorithm by the renumbering procedure, the local search heuristic and the initial solution based on the simplified model. Although, these are rather common approaches, we tailor them to our specific problem in a non-trivial and efficient way. On a final note, the presented formulation of our problem is quite general which allows for straightforward use in other potential applications.


\section*{Acknowledgements}
\label{sec:acknow}

Computational resources were supplied by the project "e-Infrastruktura CZ" (e-INFRA LM2018140) provided within the program Projects of Large Research, Development and Innovations Infrastructures.

\section*{Funding}
\label{sec:fund}

The work on this paper was supported by the Internal Grant Agency of the Prague University of Economics and Business under project F4/27/2020.


\end{document}